\documentclass[sigconf]{aamas} 

\usepackage{balance} 
\usepackage{graphicx}
\graphicspath{ {./images/} }
\usepackage{caption}
\usepackage{subcaption}
\usepackage{pgf}
\usepackage {multirow}
\usepackage [english]{babel}
\usepackage{fixltx2e}
\usepackage{setspace}
\usepackage{booktabs}
\usepackage{array}
\usepackage{tabularx}
\usepackage{calc}
\usepackage{makecell}
\usepackage{algorithm}
\usepackage{algpseudocode}
\usepackage{paralist}
\usepackage{ragged2e}

\setcopyright{ifaamas}
\acmConference[AAMAS '22]{Proc.\@ of the 21st International Conference
on Autonomous Agents and Multiagent Systems (AAMAS 2022)}{May 9--13, 2022}
{Online}{P.~Faliszewski, V.~Mascardi, C.~Pelachaud,
M.E.~Taylor (eds.)}
\copyrightyear{2022}
\acmYear{2022}
\acmDOI{}
\acmPrice{}
\acmISBN{}

\acmSubmissionID{762}

\title[S3F]{State Supervised Steering Function for Sampling-based Kinodynamic Planning}

\author{Pranav Atreya}
\affiliation{
  \institution{University of Texas at Austin}
  \city{Austin}
  \state{TX}
  \country{United States}}
\email{pranavatreya@utexas.edu}

\author{Joydeep Biswas}
\affiliation{
  \institution{University of Texas at Austin}
  \city{Austin}
  \state{TX}
  \country{United States}}
\email{joydeepb@cs.utexas.edu}

\begin{abstract}
  Sampling-based motion planners such as RRT* and BIT*, when applied to kinodynamic motion planning, rely on steering 
  functions to generate time-optimal solutions connecting sampled states. Implementing exact steering functions requires 
  either analytical solutions to the time-optimal control problem, or nonlinear programming (NLP) solvers to solve the 
  boundary value problem given the system's kinodynamic equations. Unfortunately, analytical solutions are unavailable 
  for many real-world domains, and NLP solvers are prohibitively computationally expensive, hence fast and optimal 
  kinodynamic motion planning remains an open problem. We provide a solution to this problem by introducing State 
  Supervised Steering Function (S3F), a novel approach to learn time-optimal steering functions. S3F is able to 
  produce near-optimal solutions to the steering function orders of magnitude faster than its NLP counterpart. 
  Experiments conducted on three challenging robot domains show that RRT* using S3F significantly outperforms 
  state-of-the-art planning approaches on both solution cost and runtime. We further provide a proof of probabilistic 
  completeness of RRT* modified to use S3F.
\end{abstract}


\newcommand{\BibTeX}{\rm B\kern-.05em{\sc i\kern-.025em b}\kern-.08em\TeX}

\newcommand\inputpgf[2]{{
\let\pgfimageWithoutPath\pgfimage
\renewcommand{\pgfimage}[2][]{\pgfimageWithoutPath[##1]{#1/##2}}
\input{#1/#2}
}}

\DeclareMathOperator*{\argmin}{arg\,min}

\begin{document}


\pagestyle{fancy}
\fancyhead{}


\maketitle 


\section{Introduction and Related Work}

This work tackles the kinodynamic motion planning (KDMP) problem, which is the problem of computing a kinodynamically feasible motion plan that takes a robot from an initial configuration to a goal region. We begin by formally defining the KDMP problem and then survey the various approaches to solving it.

Let $X_C$ be the configuration space of the robot. The state space $X$ is defined as the Cartesian product of $X_C$ with $X_D$, the set of dynamics variables needed to fully describe the dynamics of the robot at any given instance in time. $X_D$ typically consists of time derivatives of elements of $X_C$. Let $U$ be the control space of the robot. The kinodynamic constraints are described by the differential equation $\dot{x}(t) = f(x(t), u(t))$, where $x(t) \in X$ and $u(t) \in U$. The KDMP problem differs from the purely kinematic motion planning (KMP) problem in that the KMP problem operates only on the configuration space $X_C$. Let $X_\mathrm{obs} \in X$ be the set of obstacle-colliding states and let $X_\mathrm{free} = X \backslash X_\mathrm{obs}$ be the set of valid states. Let $x_\mathrm{init} \in X_\mathrm{free}$ be the initial state of the robot and let $X_\mathrm{goal} \subset X_\mathrm{free}$ be the goal region. The objective of the KDMP problem is to find a collision free path that takes the robot from $x_\mathrm{init}$ to $X_\mathrm{goal}$ while obeying the kinodynamic constraints. The solution to the KDMP problem is a mapping $c(t) : [0, t_f] \rightarrow U$ from time to control inputs such that applying $c(t)$ starting from the state $x_\mathrm{init}$ traces out a path $\xi(t) : [0, t_f] \rightarrow X_\mathrm{free}$ such that $\xi(t_f) \in X_\mathrm{goal}$. A motion plan is considered optimal if it minimizes some cost function $C(t_f, c, \xi)$. The time-optimal solution minimizes the total time $t_f$. 

We review the state of the art approaches to solving the KDMP problem, including search-based planning, sampling-based planning, and learning-based solutions.

\textbf{Search-based planning} typically involves constructing a state lattice $G=(V, E)$ where $V \subset X_\mathrm{free}$ and the edges $E$ are pre-defined kinodynamically feasible motion primitives~\citep{pivtoraiko2011kinodynamic}. This lattice can then be searched using any graph search algorithm to obtain a solution. Increasing the resolution of the lattice increases the chances that a solution can be found, but comes with an exponential increase in computational cost. Finding a set of motion primitives that work well can also be difficult. Search-based planning algorithms are resolution optimal, in that they can find solutions that are optimal with respect to the discretization used.

\textbf{Sampling-based planning} makes use of a continually improving discretization of the state space through random sampling. One of the most effective sampling-based planning algorithms is the Rapidly Exploring Random Tree (RRT)~\citep{lavalle2001randomized} algorithm. The RRT algorithm works by incrementally sampling the state space and extending the nearest vertex in the tree towards that sample. Because this extension can be made by a random propagation of controls, the RRT algorithm can be applied to kinodynamic systems. 

RRTs have also been integrated with machine learning approaches to solve the KDMP problem. One such work employs the k-nearest-neighbors algorithm within the RRT framework to approximate the cost-to-go function and expand vertices in the tree~\citep{wolfslag2018rrt}. It however suffers from lack of optimality of computed trajectories and is only demonstrated to work for simple environments. Reinforcement Learning RRT (RL-RRT)~\citep{chiang2019rl} trains an RL agent to do local planning and uses an RRT to guide exploration. The resulting motion plan is suboptimal and since the RL local planner is trained on particular obstacle configurations, may not generalize well to new obstacle environments. Probabilistic Roadmap RL (PRM-RL)~\citep{faust2018prm} also uses RL for local planning but maps sensor observations directly to actions and does not attempt to produce optimal trajectories. 

RRT and the aforementioned RRT based algorithms do not produce optimal solutions. An alternative algorithm that produces optimal solutions while maintaining the computational efficiency of RRT is the RRT* algorithm~\citep{karaman2011sampling}. The RRT* algorithm makes use of a rewiring step to ensure that the path from the root to any vertex in the tree is optimal with respect to the connections in the tree. Because of this, the RRT* algorithm is asymptotically optimal. Many variants of the RRT* algorithm exist that have proven to work well in practice. Informed RRT*~\citep{gammell2014informed} improves on RRT* by ensuring that after an initial solution has been found, only states that have the potential to improve the solution are considered as candidate vertices. The BIT* algorithm~\citep{gammell2015batch} integrates graph-based and sampling-based planning techniques to more efficiently find and improve on solutions to the planning problem. 

One caveat of optimal sampling-based algorithms including RRT* and BIT* is that
they all require an optimal steering function to connect states. For any two
states $x_a, x_b \in X$, a steering function $S(x_a, x_b)$ produces a trajectory
$T : [0, t_f] \rightarrow U$, a mapping from time to control inputs. Integrating
$T$ from $x_a$ according to the equation of motion $f$ produces a path $\Gamma :
[0, t_f] \rightarrow X$, a mapping from time to states. An optimal steering
function $S^*(x_a, x_b)$ produces a trajectory $T^* : [0, t_f] \rightarrow U$
and a path $\Gamma^* : [0, t_f] \rightarrow X$ that in addition to satisfying
the aforementioned constraints, satisfies $\Gamma^*(t_f) = x_b$ and minimizes
some cost function, most commonly time. There exist algorithms like
Stable-Sparse RRT (SST)~\citep{li2015sparse} and Asymptotically Optimal RRT
(AO-RRT)~\citep{hauser2016asymptotically} that do not require a steering
function, but in practice they tend to take a significant amount of time to find
good quality solutions. Analytical solutions to the steering function exist for
some robots, such as those with linear dynamics~\citep{webb2013kinodynamic}, and so do
iterative solutions for specific systems such as omnidirectional robots with bounded acceleration~\citep{balaban2018realtime},
but for most systems computing the optimal steering function requires a call to
a computationally expensive nonlinear programming (NLP) solver. There are ways
to decrease the computational overhead of NLP solvers to make planning
tractable~\citep{xie2015toward}, but the NLP solver still remains a significant
bottleneck. Previous work has explored whether the steering function can be
learned~\citep{zheng2021sampling}. The learning setup used however was unable to
connect arbitrary start and goal states, a necessity if the steering function is
to be used in an optimal sampling-based planning algorithm.

\textbf{Reinforcement learning} has also been applied to the KDMP problem. One approach to KDMP for linear systems uses continuous-time Q-learning~\citep{kontoudis2019kinodynamic} to deal with dynamics whose differential equations of motion are inaccurate or unreliable. Some have also proposed formulating the KDMP problem entirely as a Markov Decision Process (MDP), where the solution KDMP policy is learned by RL~\citep{butyrev2019deep}.

\textbf{Learning optimal control} policies is a research area that has also been recently explored. Past works~\citep{ghosh2012near}~\citep{tsiotras2014real}~\citep{sanchez2018real}~\citep{tailor2019learning} have attempted to train a neural network to learn to produce optimal controls. All of these works however keep the goal state fixed, and so a new policy would need to be learned for every goal state.

\textbf{Optimization-based planning} methods rely on numerical optimization to find a solution to the goal that minimizes some cost objective. Example works that fall under this category include GuSTO~\citep{bonalli2019gusto}, CHOMP~\citep{ratliff2009chomp}, and STOMP~\citep{kalakrishnan2011stomp}. While such optimization-based methods are effective at finding solutions given good initialization, they find difficulty in handling cases where initial solutions are unknown, or when the optimization objective function has local minima (often due to obstacles).

\textbf{Integrated planning and learning} approaches have recieved significant attention lately. Search on the Replay Buffer (SoRB)~\citep{eysenbach2019search} demonstrates how the success rate of goal-conditioned RL on long horizon tasks can be improved by adding a planning component. SoRB however is unable to provide theoretical guarantees on completeness and faces difficulty when run on unseen environments. One approach~\citep{allen2016real} uses precomputation and machine learning to enable real-time kinodynamic planning for quadrotors. It is able to avoid solving two-point boundary value problems directly on quadrotor dynamics by using minimum snap polynomial splines, a technique that only works for a limited class of systems. Model-Predictive Motion Planning Networks (MPC-MPNet)~\citep{li2021mpc} proposes the integration of multiple neural components along with Model Predictive Control to solve the kinodynamic motion planning problem. The algorithm is compared with SST and is shown to have faster planning times. It however is unable to produce lower cost paths than SST and drops in performance on unseen environments.

While many approaches exist for kinodynamic planning, none so far are able to find low cost solutions in a computationally efficient manner. Approaches either sacrifice low solution cost or performance in pursuit of the other. We propose with this work that both are attainable. In contrast to many learning approaches, our work is also agnostic to obstacle configurations, and so generalizes well to new environments. 

In summary, in this paper we contribute:
\begin{inparaenum}[1)]
  \item State Supervised Steering Function (S3F), a learning-based technique to efficiently compute the steering function required by optimal sampling-based planners;
\item S3F-RRT*, a probabilistically complete RRT* algorithm that uses S3F as its
steering function; and 
  \item Empirical results for three kinodynamically-complex robots that demonstrate that S3F-RRT* outperforms state-of-the-art kinodynamic planners.
\end{inparaenum}

\section{Kinodynamic Planning with State Supervised Steering Function}

Recall from earlier that given two arbitrary states $x_a, x_b \in X$ the optimal
steering function $S^*(x_a, x_b)$ produces a trajectory $T^*$ that optimally
connects these two states. 
We are interested in learning a function $\tilde{S}$ that approximates $S^*$
such that $\tilde{S}(x_a, x_b)$ produces a near-optimal trajectory $\tilde{T}
\simeq T^* $. The control trajectory $\tilde{T}$ can be integrated to obtain a path $\tilde{\Gamma}$. 

\subsection{Steering Function Formulation}

\begin{figure*}
  \centering
  \begin{subfigure}{.5\textwidth}
    \centering
    \scalebox{0.55}{\input{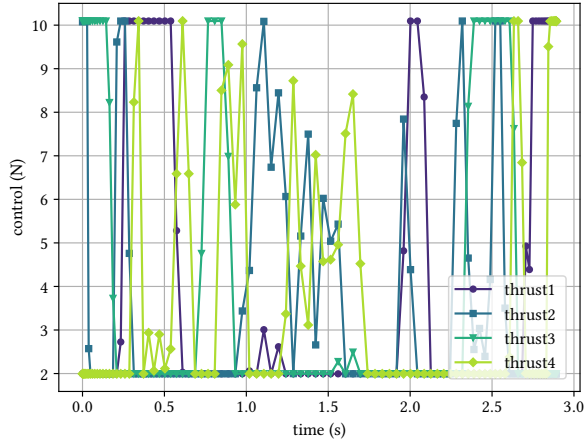}}
    \vspace{-0.6cm}
    \caption{Control function}
    \label{quad:control}
  \end{subfigure}%
  \begin{subfigure}{.5\textwidth}
    \centering
    \scalebox{0.55}{\input{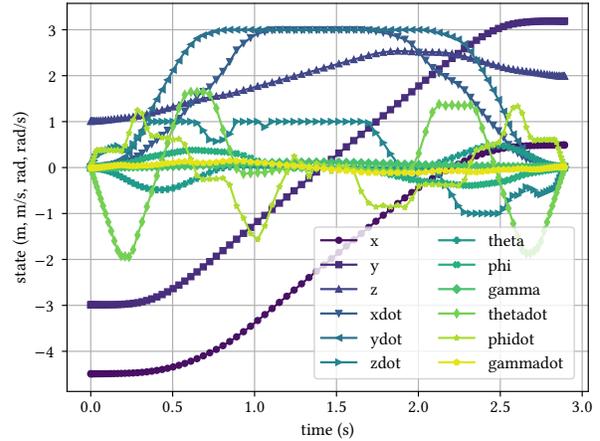}}
    \vspace{-0.6cm}
    \caption{State function}
    \label{quad:state}
  \end{subfigure}
  \vspace{-0.3cm}
  \caption{Quadrotor optimal control and state functions}
  \label{quad}
  \Description{Quadrotor optimal control and state functions}
\end{figure*}

Rather than learning $\tilde{S}$ that produces $\tilde{T}$ directly, we simplify
the learning problem by constructing $\tilde{T}$ in an iterative manner. This
can be done by learning a policy $\pi : X \times X \rightarrow U$ where $\pi$ takes
as input the current state of the robot $x_t$ and the goal state $x_b$ and
produces as output a constant-time control input $u$ to be executed for a fixed
period of time $\tau$, resulting in a new state $x_{t+1}$. Iteratively calling
$\pi$ for a fixed number of iterations $n$ results in the generation of a
piecewise constant control function that we denote $T_\mathrm{max}$. Integrating
$T_\mathrm{max}$ from the start state $x_a$ yields the state function
$\Gamma_\mathrm{max}$. 

$\tilde{T}$ can be obtained from $T_\mathrm{max}$ by discarding from $T_\mathrm{max}$ all controls past 
the time when the robot has reached the goal. To be able to do this, $n\tau$, the duration of 
$T_\mathrm{max}$, needs to be greater than the time it takes to connect any two states in $X$ 
optimally. 
The next step is to determine when $T_\mathrm{max}$ actually reaches 
the goal. The naive approach is to simply select the time at which $\Gamma_\mathrm{max}$ is closest 
to $x_b$ where closeness is defined using Euclidean distance. The problem with this approach is that 
our trajectories not only need to reach the goal but also be optimal with respect to the time to goal. 
Let’s say for one particular trajectory the robot reaches a distance $d_1$ from the goal at time $t_1$ and a 
distance $d_2$ from the goal at time $t_2$. If $d_2$ is the closest distance, then we are guaranteed 
to pick $t_2$ as our ending time, even if $d_2$ is marginally less than $d_1$. However it may be 
possible that $t_2$ is significantly greater than $t_1$, and so just to reach a little closer to 
the goal we’re sacrificing significant time optimality. This type of analysis motivates the solution 
to this problem. Since there are in essence two objectives that we are optimizing over when selecting 
the end time – distance to goal and time to reach goal – we should construct a reward function that 
fairly takes into account both. The following reward function $R(t)$ does exactly this:
\begin{equation}
\begin{aligned}
  &R(t) = \alpha \frac{||x_a-x_b||-||\Gamma_\mathrm{max}(t)-x_b||}{||x_a-x_b||}-t + R_b(\Gamma_\mathrm{max}(t), x_b) \\
  &R_b(\Gamma_\mathrm{max}(t), x_b) = \begin{cases}\beta\text{, if }||\Gamma_\mathrm{max}(t)-x_b||\leq \mu \\ 0\text{, otherwise}\end{cases} 
\end{aligned}
\end{equation}
The first term is a normalized difference of potential functions, and is 
maximized when the candidate terminal state is situated at the goal. The use of such potential functions 
was first introduced as a policy invariant mechanism for reward shaping~\citep{ng1999policy}. The second term, $-t$, takes 
into account the second objective: minimizing the time to the goal. Finally the third term provides 
an additional incentive if the candidate terminal state is very close ($\leq \mu$ distance away) to 
the goal. The hyperparameters $\alpha$, $\beta$, and $\mu$ are positive constants which can be tuned to adjust 
the relative weights of the three terms. For all time points which this reward function is calculated, the 
end time will be the time with the greatest reward. $\tilde{T}$ can then be obtained by discarding all 
control inputs in $T_\mathrm{max}$ after the end time.  
\vspace{-0.6cm}

\subsection{Learning the Policy}

The previous section showed how the steering function $\tilde{S}$ can be
constructed from a learned policy $\pi$. We next present how $\pi$ is learned.

We employ a supervised learning approach to learn $\pi$. 
Since
the end goal is to learn the optimal steering function, our dataset consists of
solutions to the optimal steering function for a large number of start and goal
states. This dataset, generated by an NLP solver, consists of a series of
trajectories each described by a tuple $(T^*, \Gamma^*, t_f)$. Here $T^* : [0,
t_f] \rightarrow U$ and $\Gamma^* : [0, t_f] \rightarrow X$ are the control and
state functions introduced earlier. We used the PSOPT~\citep{becerra2010solving}
optimal control library to generate the trajectories. The start and goal states
for each trajectory in the dataset are sampled uniformly at random from the full
state space to ensure good state space coverage.  

To learn $\pi$ using this dataset we employ the fact that $\pi$ is used to
generate control and state functions $\tilde{T}$ and $\tilde{\Gamma}$. The
arguably simplest approach is to have $\tilde{T}$ imitate $T^*$ for each start
and goal pair in the dataset. The discrepancy in the fact that $\tilde{T}$ is
piecewise constant whereas $T^*$ is continuous can be accounted for by simply
averaging controls in $T^*$ at each length $\tau$ time interval. $\pi$ would
then be directly supervised by the averaged constant controls in $T^*$. Although
straightforward, this approach fails to learn a well-performing policy. The
primary reason for this is that the learning problem involves the approximation
of a highly discontinuous function. $\pi$ is tasked with learning the optimal
control function which for many kinodynamic systems is a bang-bang control
function. Figure \ref{quad:control} shows an example of this for the quadrotor
robot -- such discontinuous control functions are hard to represent and learn
directly, even by supervised learning.

\algnewcommand\algorithmicforeach{\textbf{for each}}
\algdef{S}[FOR]{ForEach}[1]{\algorithmicforeach\ #1\ \algorithmicdo}
\begin{figure}
  \hrulefill \\
  \vspace{-0.05cm}
  \textbf{S3F-RRT*\texttt{()}~~~~~~~~~~~~~~~~~~~~~~~~~~~~~~~~~~~~~~~~~~~~~~~~~~~~~~~~~~~~~~~~~~~~~~~~~~~~~~~~~~~~~~~~~~~~~~~~~~~~~~~~~~~~~~~~~~~~~} \\
  \vspace{-0.25cm}
  \hrulefill
  \vspace{-0.1cm}
  \begin{algorithmic}[1]
    \State $V \leftarrow \{x_\mathrm{init}\}, E \leftarrow \emptyset$
    \For{$i = 1..n$}
      \State $x_\mathrm{rand} \leftarrow \texttt{SampleFree()}$
      \State $x_\mathrm{parent} \leftarrow \varnothing, x_\mathrm{ext} \leftarrow \varnothing, c_\mathrm{min} \leftarrow \infty$
      \State $X_\mathrm{near} \leftarrow \texttt{NearTo(}G=(V,E), x_\mathrm{rand}\texttt{)}$
      \ForEach{$x \in X_\mathrm{near}$}
        \State $T \leftarrow \texttt{Steer(}x, x_\mathrm{rand}\texttt{)}$
        \State $x_\mathrm{new} \leftarrow \texttt{EndState(}x, T\texttt{)}$
        \State $c_\mathrm{traj} \leftarrow \texttt{SteeringCost(}T\texttt{)}$
        \State $b \leftarrow \texttt{Dist(}x_\mathrm{new}, x_\mathrm{rand}\texttt{)} < r_\mathrm{error} \land \texttt{ObstacleFree(}x, T\texttt{)}$
        \If{$\texttt{Cost(}x\texttt{)} + c_\mathrm{traj} < c_\mathrm{min} \land b$}
          \State $x_\mathrm{parent} \leftarrow x$
          \State $x_\mathrm{ext} \leftarrow x_\mathrm{new}$
          \State $c_\mathrm{min} \leftarrow \texttt{Cost(}x\texttt{)} + c_\mathrm{traj}$
        \EndIf
      \EndFor
      \If{$c_\mathrm{min} \neq \infty$}
        \State $V \leftarrow V \cup \{x_\mathrm{ext}\}$
        \State $E \leftarrow E \cup \{(x_\mathrm{parent}, x_\mathrm{ext})\}$
      \EndIf
      \State $\texttt{Rewire(}V, E, x_\mathrm{ext}\texttt{)}$
    \EndFor
    \State \Return{$G=(V,E)$}
  \end{algorithmic}

\vspace{0.5cm}
\hrulefill \\
\vspace{-0.05cm}
\textbf{Rewire\texttt{(}$V$, $E$, $x_\mathrm{ext}$\texttt{)}~~~~~~~~~~~~~~~~~~~~~~~~~~~~~~~~~~~~~~~~~~~~~~~~~~~~~~~~~~~~~~~~~~~~~~~~~~~~~~~~~~~~~~} \\
\vspace{-0.25cm}
\hrulefill
\vspace{-0.1cm}
  \begin{algorithmic}[1]
    \State $X_\mathrm{near} \leftarrow \texttt{NearFrom(}G=(V,E), x_\mathrm{ext}\texttt{)}$
    \ForEach{$x \in X_\mathrm{near}$}
      \State $T \leftarrow \texttt{Steer(}x_\mathrm{ext}, x\texttt{)}$
      \State $x_\mathrm{new} \leftarrow \texttt{EndState(}x_\mathrm{ext}, T\texttt{)}$
      \State $c_\mathrm{traj} \leftarrow \texttt{SteeringCost(}T\texttt{)}$
      \State $b \leftarrow \texttt{Dist(}x_\mathrm{new}, x\texttt{)} < r_\mathrm{error} \land \texttt{ObstacleFree(}x_\mathrm{ext}, T\texttt{)}$
      \If{$\texttt{Cost(}x_\mathrm{ext}\texttt{)} + c_\mathrm{traj} < \texttt{Cost(}x\texttt{)} \land b$}
        \State $V \leftarrow V \backslash \{x\} \cup \{x_\mathrm{new}\}$
        \State $E \leftarrow E \backslash \{(\texttt{Parent(}x\texttt{)}, x)\} \cup \{(x_\mathrm{ext}, x_\mathrm{new})\}$
        \State $\texttt{PropagateRewiring(}x, x_\mathrm{new}\texttt{)}$
      \EndIf
    \EndFor
  \end{algorithmic}

\vspace{0.5cm}
\hrulefill \\
\vspace{-0.05cm}
\textbf{PropagateRewiring\texttt{(}$x$, $x_\mathrm{new}$\texttt{)}~~~~~~~~~~~~~~~~~~~~~~~~~~~~~~~~~~~~~~~~~~~~~~~~~~~~~~~~~~~~~~~~~} \\
\vspace{-0.25cm}
\hrulefill
\vspace{-0.1cm}
  \begin{algorithmic}[1]
    \ForEach{$x_\mathrm{child} \in \texttt{Children(}x\texttt{)}$}
      \State $T \leftarrow \texttt{Trajectories(}x, x_\mathrm{child}\texttt{)}$
      \If{$\texttt{ObstacleFree(}x_\mathrm{new}, T\texttt{)}$}
        \State $x_\mathrm{next} \leftarrow \texttt{EndState(}x_\mathrm{new}, T\texttt{)}$
        \State $V \leftarrow V \backslash \{x_\mathrm{child}\} \cup \{x_\mathrm{next}\}$
        \State $E \leftarrow E \backslash \{(x, x_\mathrm{child})\} \cup \{(x_\mathrm{new}, x_\mathrm{next})\}$
        \State $\texttt{PropagateRewiring(}x_\mathrm{child}, x_\mathrm{next}\texttt{)}$
      \Else
        \State $\texttt{DeleteSubtree(}x_\mathrm{child}\texttt{)}$
      \EndIf
    \EndFor
  \end{algorithmic}
\caption{S3F-RRT* Algorithm}
\label{algorithm}
\Description{S3F-RRT* Algorithm}
\end{figure}

The solution to this problem is to not use the optimal control function $T^*$ to
supervise the learning, but to instead use the optimal state function
$\Gamma^*$. We term this approach State Supervised Steering Function (S3F). Due
to the differential equation $f$ that defines the kinodynamic constraints, state
functions are guaranteed to be differentiable (and thus continuous), making
learning the optimal state function a feasible problem. Figure~\ref{quad:state}
shows an example of such a state function for the quadrotor robot -- note that
despite the associated control function (Figure~\ref{quad:control}) being
discontinuous, the state function is smooth and continuous. The goal now
is to have $\tilde{\Gamma}$ imitate $\Gamma^*$ for each trajectory in the
dataset. This can be done by ensuring that for various time points $t$ in the
range $[0, t_f]$, $\tilde{\Gamma}(t)=\Gamma^*(t)$. Recall that $\tilde{\Gamma}$
is only obtained by integrating $\tilde{T}$. This can be accounted for with the
following procedure: sample a series of time points $(t_0 ... t_k)$ in the range
$[0, t_f - \tau]$. For each time point $t$, assume that the robot is currently
at $\Gamma^*(t)$. If $\tilde{\Gamma}$ is to imitate $\Gamma^*$, the robot should
be at $\Gamma^*(t+\tau)$ at time $t + \tau$. The actual location of the robot at
this time under the current policy $\pi$ can be calculated by evaluating
$F(\Gamma^*(t), \pi(\Gamma^*(t), x_{t_f}))$ where $x_{t_f}$ is the goal state of
the trajectory and $F : X \times U \rightarrow X$ is an integration function
that given a current state and a constant control, integrates the differential
equation of motion $f$ to compute the state $\tau$ units of time later. To get
$\tilde{\Gamma}$ to imitate $\Gamma^*$ we can thus optimize the following
learning objective:
\begin{equation}
  \argmin_\theta \sum_{\Gamma^*\in D} \sum_{t\in(t_0, ..., t_k)} [F(\Gamma^*(t), \pi(\Gamma^*(t), x_{t_f}))-\Gamma^*(t+\tau)]^2
\end{equation}
where $\theta$ is the parameter set of $\pi$ and $D$ is the dataset of optimal trajectories. The key takeaway from this learning procedure is that we are learning $\pi$ indirectly. $\pi$ is a component of a state function that we are training to be optimal, and by learning this state function we are indirectly learning the control function $\pi$. 

\begin{figure*}
  \centering
  \includegraphics[width=7.0in, height=2.0in]{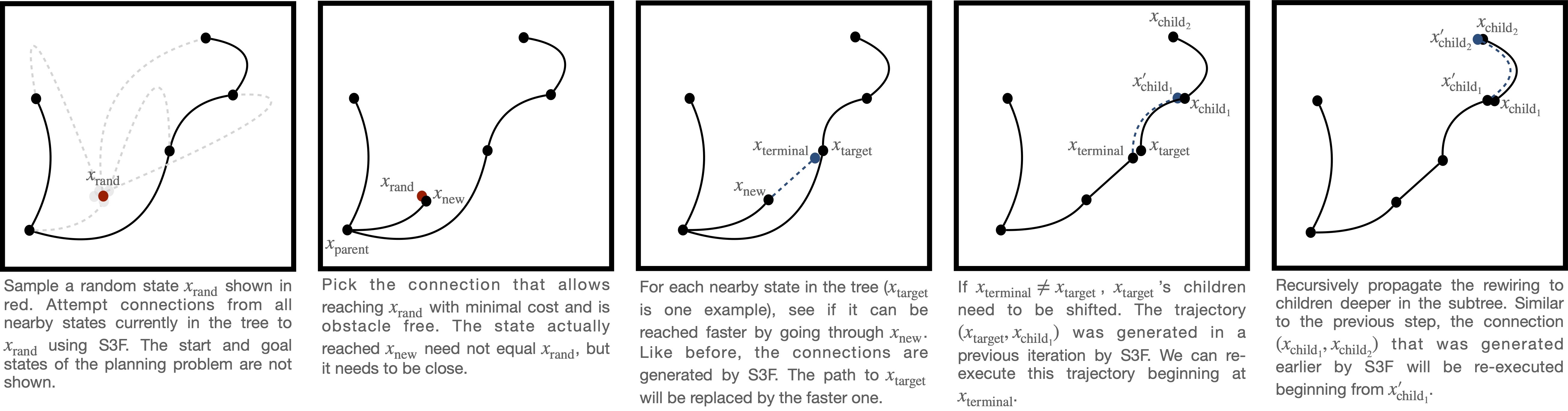}
  \vspace{-7mm}
  \caption{Illustration of the steps that take place in one iteration of S3F-RRT*}
  \label{illustration}
  \Description{Illustration of the steps that take place in one iteration of S3F-RRT*}
\end{figure*}

\subsection{Sampling-based Planning With Learned Steering Functions: S3F-RRT*}

We present S3F-RRT*, a sampling-based planning algorithm that uses the learned
steering function to solve the optimal kinodynamic motion planning problem. S3F-RRT* uses S3F as
the steering function, and employs a modified rewiring procedure to
overcome any potential local inaccuracies in S3F's trajectories.
Figure~\ref{algorithm} presents the algorithmic formulation of S3F-RRT*. Figure 
\ref{illustration} shows a visualization of what goes on in each S3F-RRT* iteration. Each
iteration begins by sampling a random collision-free state $x_\mathrm{rand}$.
The $\texttt{NearTo}$ function is then called to obtain the set of all vertices
in the current RRT* tree that are near $x_\mathrm{rand}$. A state is considered
to be near $x_\mathrm{rand}$ if the time of the optimal trajectory from that
state to $x_\mathrm{rand}$ is below some threshold.
Each
state in $X_\mathrm{near}$ is then evaluated as a possible parent to
$x_\mathrm{rand}$. $\texttt{Steer}(x, x_\mathrm{rand})$ invokes S3F to compute a control function $T$ that connects $x$ to
$x_\mathrm{rand}$. 
To determine $x_\mathrm{new}$, where the trajectory actually ends,
$\texttt{EndState}(x, T)$ integrates $T$ from $x$. $\texttt{SteeringCost}(T)$ returns
the cost of the trajectory $T$, which for a time-optimal planning problem is
simply the duration of $T$. 
$\texttt{Cost}(x)$ returns the cost of going from the start state to $x$ in the
current RRT* tree.  The
$\texttt{Dist}$ function returns the Euclidean distance between two states and
is used to ensure that the terminal state of the trajectory is close enough to
the target state. $\texttt{ObstacleFree}(x, T)$ integrates the control function
$T$ beginning at $x$ to obtain a state function that maps time to states.
$\texttt{ObstacleFree}$ then ensures that every state in this state function
does not collide with obstacles. 

After the best parent has been found and the state has been added to the tree, the rewiring procedure is invoked. Here, the set $X_\mathrm{near}$ is constructed by calling $\texttt{NearFrom}(G=(V, E), x_\mathrm{ext})$. The difference between $\texttt{NearFrom}$ and $\texttt{NearTo}$ is that $\texttt{NearFrom}(G=(V, E), x_\mathrm{ext})$ considers connections from $x_\mathrm{ext}$ to other states as opposed to from other states. $\texttt{Parent}(x)$ returns the parent of $x$ in the current RRT* tree.

The rewiring procedure internally calls $\texttt{PropagateRewiring}$. $\texttt{Children}(x)$ returns the set of all children states to $x$ in the current RRT* tree. $\texttt{Trajectories}(x, x_\mathrm{child})$ returns the control function that was computed earlier by S3F to connect $x$ and $x_\mathrm{child}$. 

One of the key differences between this algorithm and the original RRT* algorithm is the absence in this algorithmic formulation of finding the nearest state. In the original RRT* algorithm, after a state is randomly sampled, the nearest state in the tree is selected as a source of expansion. A new state is obtained by extending the nearest state towards the randomly sampled state up to a distance $\eta$, and the resultant state is used as the target for the subsequent steering function evaluations. We entirely eliminate this component of the algorithm for simplicity, a modification that was first proposed in Kinodynamic RRT*~\citep{webb2013kinodynamic}. This modification is known to not hurt theoretical asymptotic optimality of the RRT* algorithm. The main other difference in this algorithm is a series of modifications that deal with the fact that the learned steering function will reach within an error radius of the goal state. Notable among these is the existence of the $\texttt{PropagateRewiring}$ procedure.
\vspace{-5mm}

\subsection{Correctness of S3F-RRT*}

There are two criteria for correctness:
solutions returned by S3F-RRT* must satisfy the kinodynamic constraints and must avoid obstacles. Any 
operation on the S3F-RRT* tree (such as rewiring) can be reformulated as a sequence of state addition 
and state deletion operations. State deletion by default cannot violate correctness. 
State addition also satisfies correctness because (1) a state is only added to the tree if the path from 
the parent to the state is collision free and (2) the path from the parent to the state is generated by 
integrating the differential equation of motion, implying that the path to the state satisfies kinodynamic 
constraints. Thus S3F-RRT* is correct.

\subsection{Probabilistic Completeness Proof of S3F-RRT*}

\begin{figure*}
  \centering
  \begin{subfigure}[t]{2.0in}
    \centering
    \includegraphics[height=2.0in]{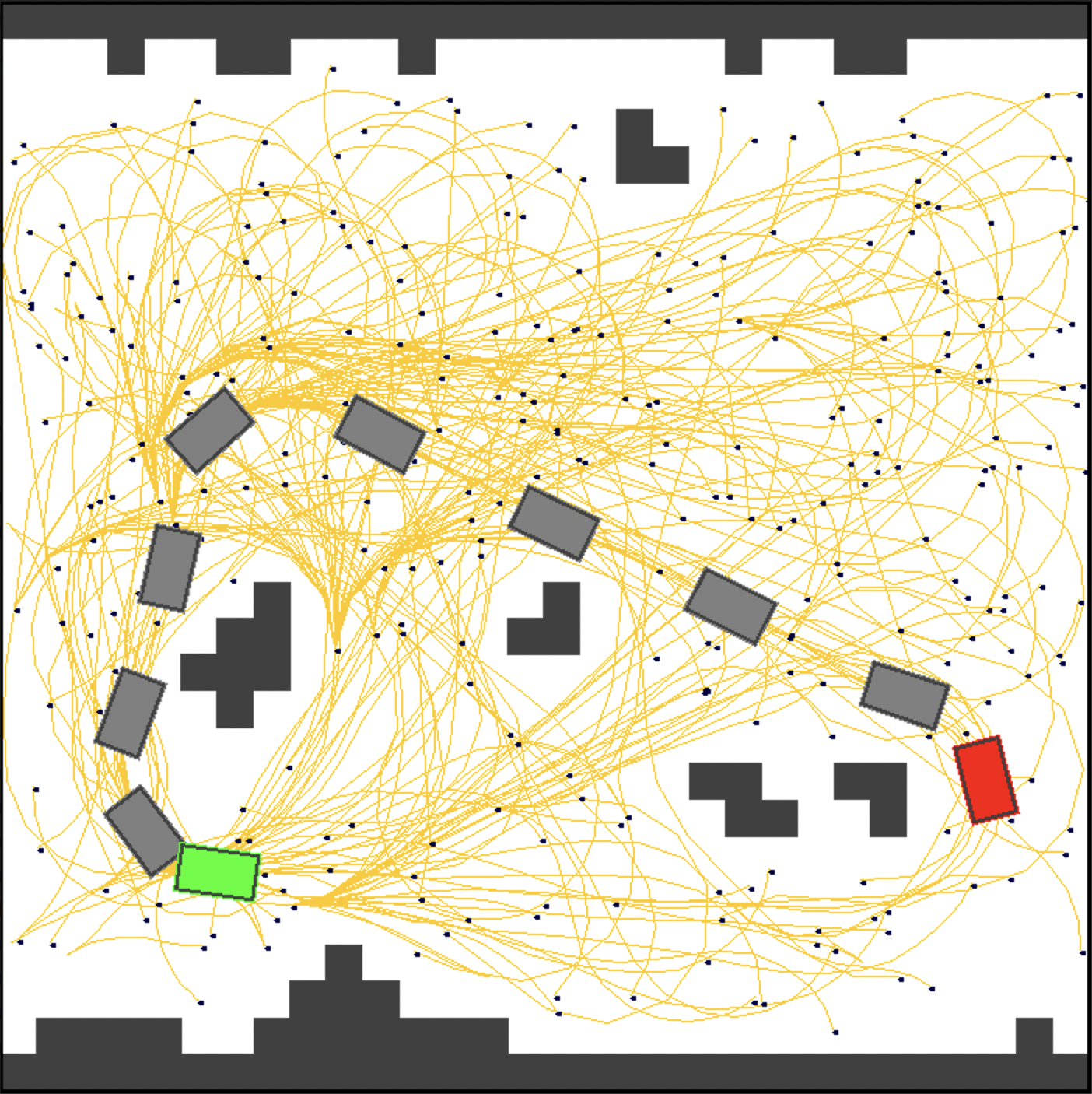}
    \caption{Dubin's Car with Acceleration}
  \end{subfigure}%
  \hspace{0.5cm}
  ~
  \centering
  \begin{subfigure}[t]{2.0in}
    \centering
    \includegraphics[height=2.0in]{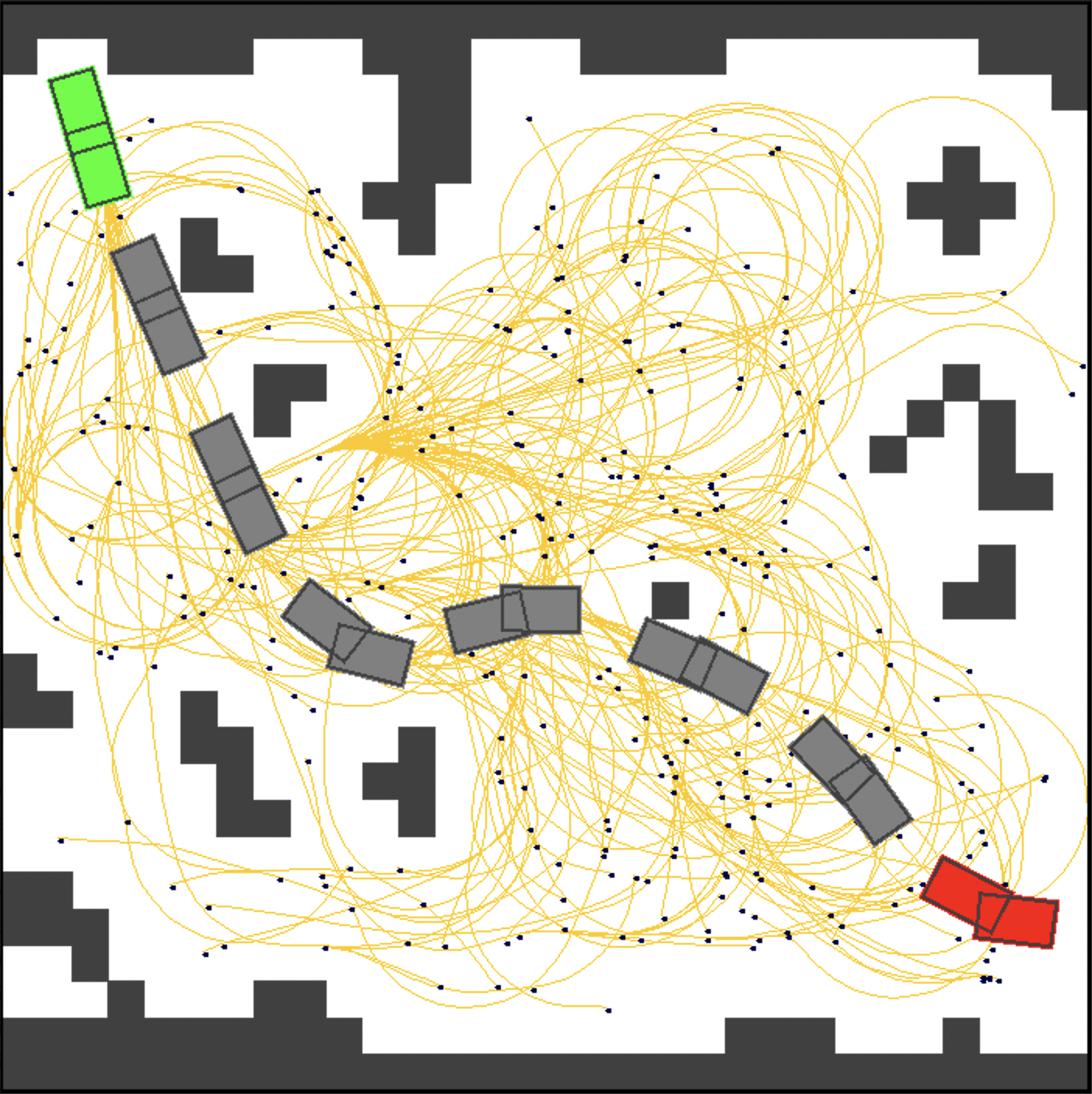}
    \caption{Tractor Trailer}
  \end{subfigure}%
  \hspace{0.5cm}
  ~
  \centering
  \begin{subfigure}[t]{2.0in}
    \centering
    \includegraphics[height=2.0in]{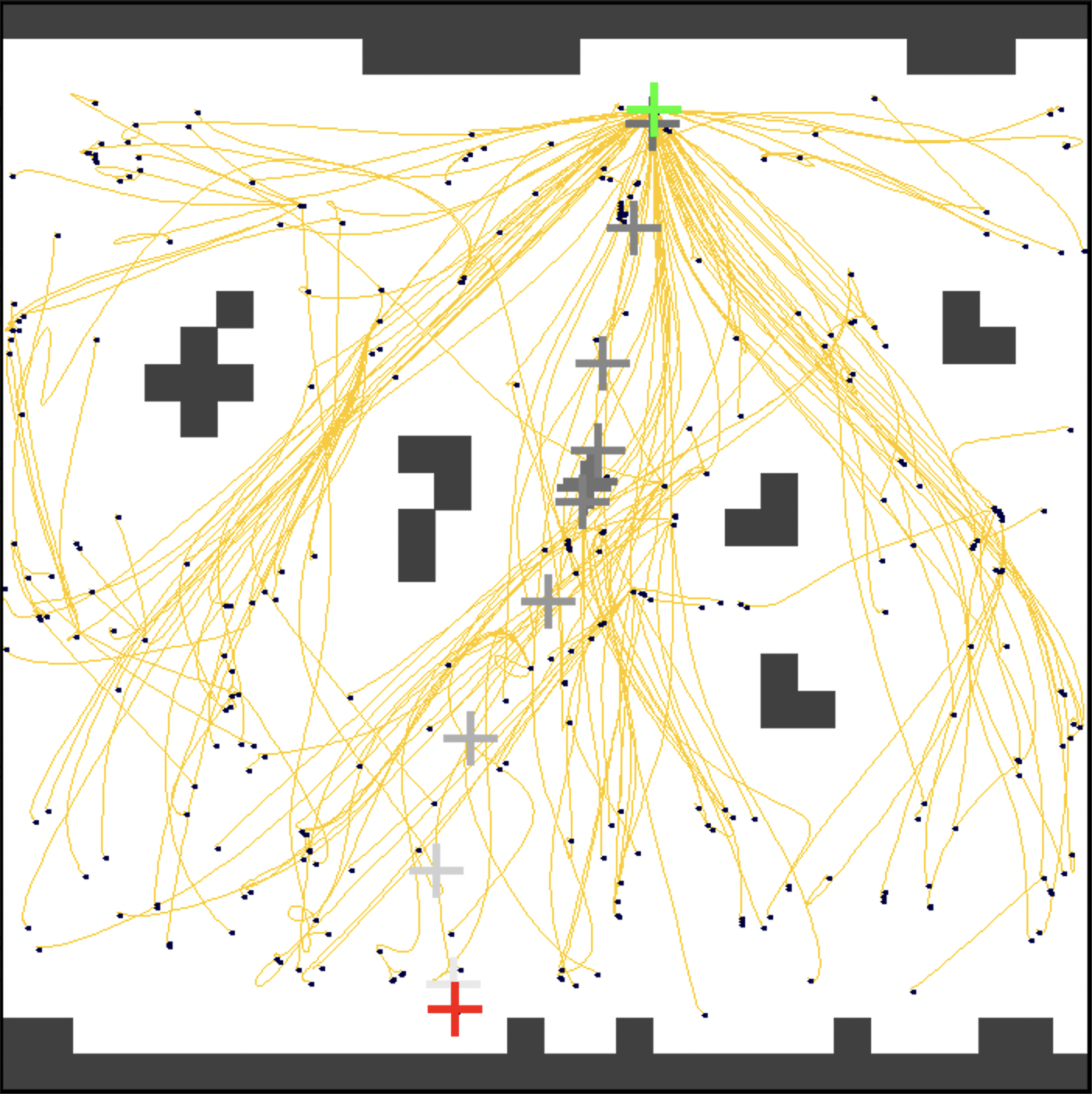}
    \caption{Quadrotor}
  \end{subfigure}
  
  \vspace{-0.2cm}
  \caption{Sample planning trees after running S3F-RRT* on the three robot domains. The best solution found from the (green) start state to the (red) goal state is shown explicitly. A large spacing between consecutive gray states indicates a high velocity. In the planning trees, the dots are the vertices of the tree and the orange connections are the edges. In (c), dark gray states are of low elevation and light gray states are of high elevation. }
  \label{sample_solns}
  \Description{Sample planning trees after running S3F-RRT* on the three robot domains. The best solution found from the (green) start state to the (red) goal state is shown explicitly. A large spacing between consecutive gray states indicates a high velocity. In the planning trees, the dots are the vertices of the tree and the orange connections are the edges. In (c), dark gray states are of low elevation and light gray states are of high elevation. }
\end{figure*}

Here we present a summary of the proof of probabilistic completeness (PC) of the S3F-RRT* algorithm. S3F-RRT* is a 
	modification of the original RRT* algorithm~\citep{karaman2011sampling} designed to make use of a 
	learned steering function. The proof largely follows the structure of the proof of probabilistic
	completeness of geometric RRT~\citep{kleinbort2018probabilistic}, though significant modifications have 
	been made to take into account the presence of kinodynamic constraints and the use of a learned steering 
	function. The full proof can be found in the supplementary materials.

  Let $c^*(x_a, x_b)$ denote the cost of the optimal trajectory from 
  $x_a$ to $x_b$, or equivalently the kinodynamic distance from $x_a$ to $x_b$. We assume that $c^*$ obeys the triangle inequality, that is, $c^*(x_a, x_b) \leq c^*(x_a, x) + c^*(x, x_b)$ for all 
  $x \in X$. Let $\tilde{S}$ be a learned steering function. We assume that with nonzero probability $p$, $\tilde{S}(x_a, x_b)$ yields 
  a state function $\tilde{\Gamma}$ that satisfies $c^*(\tilde{\Gamma}(t), x_b) \leq c^*(x_a, x_b)$ for all $t \in [0, t_f]$. This 
  assumption in essence states that every state along the path produced by $\tilde{S}$ is kinodynamically closer to the goal state than the start state 
  is. For a steering function trained to be optimal, this is a reasonable assumption.    
  
  We will use $B_r(x)$ to denote the subset of the state space $X$ defined by $\{x'|c^*(x', x) \leq r\}$. For simplicity, we assume 
  that there exist $\delta_\mathrm{goal} > 0, x_\mathrm{goal} \in X_\mathrm{goal}$ such that 
  $B_{\delta_\mathrm{goal}}(x_\mathrm{goal}) \subseteq X_\mathrm{goal}$. We denote this simplified goal region $B_{\delta_\mathrm{goal}}(x_\mathrm{goal})$ as $X_\mathrm{goal}^*$. The goal of the motion planning problem is to find a kinodynamically 
  feasible path $\pi : [0, t_\pi] \rightarrow X_\mathrm{free}$ such that $\pi(0) = x_\mathrm{init}$ and $\pi(t_\pi) \in X_\mathrm{goal}^*$. 
  The clearance of $\pi$ is the maximal $\delta_\mathrm{clear}$ such that $B_{\delta_\mathrm{clear}}(\pi(t)) \in X_\mathrm{free}$ for 
  all $t \in [0, t_\pi]$.

  We assume for this proof that there exists a valid trajectory $\pi : [0, t_\pi] \rightarrow 
  X_\mathrm{free}$ with clearance $\delta_\mathrm{clear} > 0$. Without loss of generality, assume 
  that $\pi(t_\pi) = x_\mathrm{goal}$, i.e., the trajectory terminates at the center of the goal 
  region. Let $L$ be the total cost of $\pi$, and let $v = min(\delta_\mathrm{clear}, 
  \delta_\mathrm{goal})$. Let $m = \frac{3L}{v}$. Define a sequence of $m+1$ points $x_0 = 
  x_\mathrm{init}, ..., x_m = x_\mathrm{goal}$ along $\pi$ such that the cost of traversal from one 
  point to the next is $\frac{v}{3}$. Therefore, $c^*(x_i, x_{i+1}) \leq \frac{v}{3}$ for every $0 
  \leq i < m$. We will now prove that as the number of iterations increases, the S3F-RRT* algorithm will generate a 
  path passing through the vicinity of these $m+1$ points with probability asymptotically approaching one. 

  \begin{lemma}
    \label{lemma}
      Suppose that S3F-RRT* has reached $B_\frac{v}{3}(x_i)$, that is, its tree contains a vertex $x_i'$ such that $x_i' \in B_\frac{v}{3}(x_i)$. If $x_\mathrm{rand} \in B_\frac{v}{3}(x_{i+1})$ and $c^*(x_i, x_\mathrm{rand}) \leq \frac{v}{3}$ (equivalently $x_i \in B_\frac{v}{3}(x_\mathrm{rand})$), then the path from the nearest neighbor $x_\mathrm{near}$ to $x_\mathrm{rand}$ lies entirely in $X_\mathrm{free}$ with probability $p$.
    \end{lemma}
    
    \begin{proof}
      See supplementary materials.
    \end{proof}
    
    \begin{theorem}
      \label{theorem}
      The probability that S3F-RRT* fails to reach $X_\mathrm{goal}^*$ from $x_\mathrm{init}$ after $k$ iterations is at most $ae^{-bk}$, for some constants $a, b \in \mathbb{R}_{>0}$. 
    \end{theorem}
    
    \begin{proof}
      See supplementary materials for full proof of Theorem \ref{theorem}. Here we present an overview.
      Assume that $B_\frac{v}{3}(x_i)$ already contains an S3F-RRT* vertex. Let $r_i$ be the probability 
      that in the next iteration a S3F-RRT* vertex will be added to $B_\frac{v}{3}(x_{i+1})$. The proof 
      in essence relies on the fact that with Lemma \ref{lemma} in place, it can be shown that the 
      probability $r_i$ is nonzero and is independent of the number of S3F-RRT* iterations $k$. In order 
      for the S3F-RRT* algorithm to reach $X_\mathrm{goal}^*$ from $x_\mathrm{init}$, a S3F-RRT* vertex 
      must be added to $B_\frac{v}{3}(x_{i+1})$ $m$ times for $0 \leq i < m$. If we let $r$ be the minimum 
      of the transition probabilities $\{r_i | \forall i (0 \leq i < m)\}$, reaching the goal can be 
      described as $k$ Bernoulli trials with success probability $r$, where the goal is reached after $m$ successful 
      outcomes. With this formulation it can be shown that the probability the goal is not reached decays 
      to zero exponentially with $k$, and thus S3F-RRT* is probabilistically complete. 
    \end{proof}



\section{Experimental Results}

We compared S3F to the current state of the art on three challenging problem spaces: Dubin’s car with acceleration, tractor trailer, and quadrotor robots. For each problem space, we solve a series of minimum-time motion planning problems using S3F-RRT*, RRT* using NLP for steering, RRT, and SST. The BARN dataset~\citep{perille2020benchmarking} was used to obtain realistic, obstacle dense maps to run the comparisons on. Figure \ref{sample_solns} depicts sample solutions and their planning trees found by S3F-RRT* on the three problem spaces.

\subsection{Robot Kinodynamics}

The three robot models used in this paper are the Dubin’s car with acceleration, tractor trailer, and quadrotor robots. Here we introduce these robot domains in more detail along with their equations of motion.
\vspace{0.2cm} \\
\textbf{Dubin’s Car with Acceleration}: $X = [x, y, \theta, v], U = [a, k]$ 
\begin{equation}
  \begin{aligned}
    \dot x&=v \cos(\theta) \hspace{1.5cm} & \dot y&=v\sin(\theta) \\
    \dot \theta &= vk      \hspace{1.5cm}     & \dot v&=a
  \end{aligned}
\end{equation}

\noindent The Dubin’s car with acceleration is a curvature constrained robot car. $x$, $y$, $\theta$, and $v$ are the $x$-position, $y$-position, orientation, and velocity of the car, and $a$ and $k$ are the acceleration and curvature control inputs. The motion of the car is subject to the curvature constraint $|k| \leq |\frac{1}{r_\mathrm{min}}|$ where $r_\mathrm{min}$ is the minimum radius of turning. 

\begin{figure*}
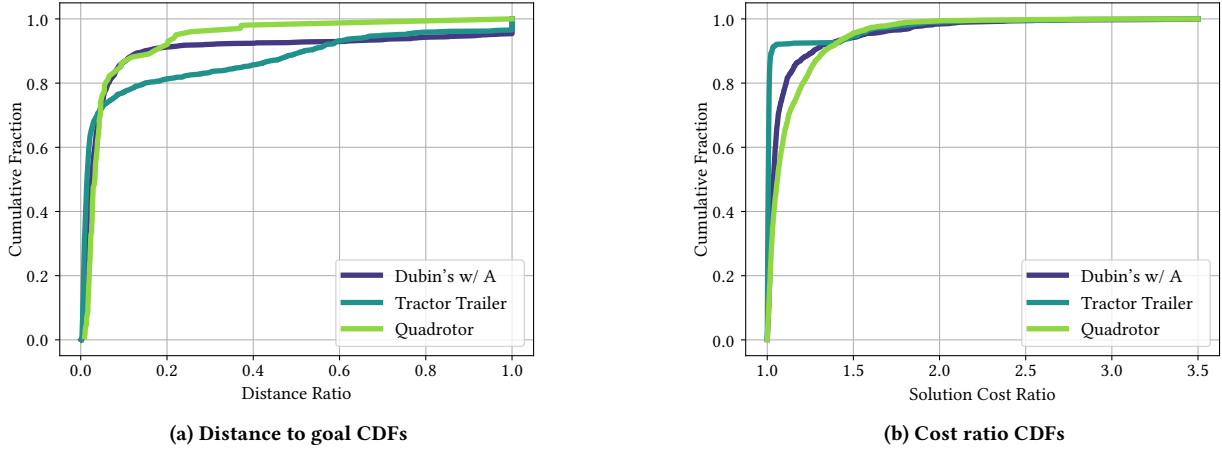

  \centering
  \begin{subfigure}[t]{3.3in}
    \centering
    \scalebox{0.5}{\input{./images/dist_cdf.pgf}}
    \caption{Distance to goal CDFs}
    \label{dist_cdf}
  \end{subfigure}%
  \hspace{0.5cm}
  ~
  \centering
  \begin{subfigure}[t]{3.3in}
    \centering
    \scalebox{0.5}{\input{./images/cost_cdf.pgf}}
    \caption{Cost ratio CDFs}
    \label{cost_cdf}
  \end{subfigure}%
  \hspace{0.5cm}
  \vspace{-0.2cm}
  \caption{CDFs of (a) the distance remaining to the goal for the three robot domains and (b) the ratios of costs of S3F's solutions over NLP's solutions for the three robot domains. Plots depict 1500 data points.}
  \Description{CDFs of (a) the distance remaining to the goal for the three robot domains and (b) the ratios of costs of S3F's solutions over NLP's solutions for the three robot domains. Plots depict 1500 data points.}
\end{figure*}

\vspace{0.2cm}
\noindent \textbf{Tractor Trailer}: $X = [x, y, \theta, v, \alpha], U = [a, \phi]$ 
\begin{equation}
  \begin{aligned}
    \dot x&=v\cos(\theta) \hspace{1.5cm}& \dot v&=a \\
    \dot y&=v\sin(\theta) \hspace{1.5cm}& \dot \alpha&=(\frac{v}{D})\sin(\theta-\alpha) \\
    \dot \theta&=(\frac{v}{L})\tan(\phi) 
  \end{aligned} 
\end{equation}
The tractor trailer robot consists of a four wheeled robot car pulling a two wheeled trailer. The robot car in isolation has the same dynamics as the Dubin’s car with acceleration. $x$, $y$, $\theta$, $v$, and $\alpha$ are the $x$-position, $y$-position, orientation, of the car, velocity of the car, and orientation of the trailer, respectively. The control inputs are $a$ and $\phi$ which represent the acceleration and heading. $L$ is the distance between the front and rear axles of the robot car, and $D$ is the length of the rod connecting the trailer with the car. 

\vspace{0.2cm}
\noindent \textbf{Quadrotor}: $X = [x, y, z, \dot x, \dot y, \dot z, \theta, \phi, \gamma, \dot \theta, \dot \phi, \dot \gamma], U = [\tau_1, \tau_2, \tau_3, \tau_4]$ 
\vspace{-0.3cm}
\begin{equation}
\begin{aligned}
  \ddot x &= \frac{1}{w}(\cos \theta \sin \phi \cos \gamma + \sin \theta \sin \gamma)(\tau_1+\tau_2+\tau_3+\tau_4) \\
  \ddot y &= \frac{1}{w}(\cos \theta \sin \phi \sin \gamma - \sin \theta \cos \gamma)(\tau_1+\tau_2+\tau_3+\tau_4)  \\
  \ddot z &= \frac{1}{w}(\cos \theta \cos \phi)(\tau_1+\tau_2+\tau_3+\tau_4)  \\
  \ddot \theta &= \frac{L(\tau_1 - \tau_3) - 2wL^2 \dot \phi \dot \gamma}{2wr^2 / 5 + 2wL^2} \\
  \ddot \phi &= \frac{L(\tau_2 - \tau_4) + 2wL^2 \dot \theta \dot \gamma}{2wr^2 / 5 + 2wL^2} \\
  \ddot \gamma &= \frac{b(\tau_1-\tau_2+\tau_3-\tau_4)}{2wr^2 / 5 + 4wL^2}
\end{aligned}
\end{equation}
The quadrotor is a lightweight, agile robot heavily used in research and industrial applications. $x$, $y$, and $z$ represent the Cartesian coordinates of the quadrotor. $\theta$, $\phi$, and $\gamma$ represent the pitch, roll, and yaw, respectively. $w$ is the weight of the quadrotor, $L$ is the length of an arm, $r$ is the radius of the sphere representing the center blob of the quadrotor, $g$ is the gravitational acceleration, and $b$ is a constant. $\tau_1$ through $\tau_4$ represent the thrusts generated by each of the four motors and are the control inputs for the quadrotor. 

\subsection{S3F Evaluation}

We evaluate the learned steering function for each of the three problem spaces on its ability to consistently reach the goal and on the time optimality of its solutions. 

We measured the former by computing for $1500$ steering function queries how much of the initial distance between the start and goal states was not traversed in the produced trajectory. Mathematically this is expressed by $\frac{d_f}{d_s}$ where $d_s$ is the distance from the start to the goal and $d_f$ is the distance from the end state of the trajectory produced by S3F to the goal. A value of $0$ indicates that the goal is reached exactly. Figure \ref{dist_cdf} depicts the cumulative distribution function (CDF) plot of $1500$ evaluations of this expression. To list a few numbers, we see that for the Dubin’s car with acceleration problem space, $85\%$ of the trajectories are within $10\%$ of $d_s$ to the goal; for the tractor trailer problem space, $75\%$ of the trajectories are within $10\%$ of $d_s$ to the goal; and for the quadrotor problem space, $85\%$ of the trajectories are within $10\%$ of $d_s$ to the goal. These results indicate that on average, S3F is able to reach very close to the desired goal.

Measuring the quality of the solutions produced by S3F in terms of time optimality can easily be done by comparing S3F’s trajectory costs with the optimal costs as determined by the NLP solver. Figure \ref{cost_cdf} shows the CDF plots of the ratios of the cost of solutions of trajectories produced by S3F with the cost of solutions of trajectories produced by the NLP solver. An ideal value of the ratio is close to $1$. Results are depicted for $1500$ trajectories. We can see that for all three problem spaces the trajectories are very close to optimal. Specifically, for the Dubin’s car with acceleration problem space, $90\%$ of S3F’s trajectories have costs that are less than $1.25$ times as suboptimal as the optimal cost; for the tractor trailer problem space, $90\%$ of S3F’s trajectories have costs that are less than $1.25$ times as suboptimal as the optimal cost; and for the quadrotor problem space, $80\%$ of S3F’s trajectories have costs that are less than $1.25$ times as suboptimal as the optimal cost. 

\subsection{Planning Comparisons}

\begin{figure}
  \centering
  \begin{subfigure}[t]{3.5in}
    \centering
    \scalebox{0.5}{\input{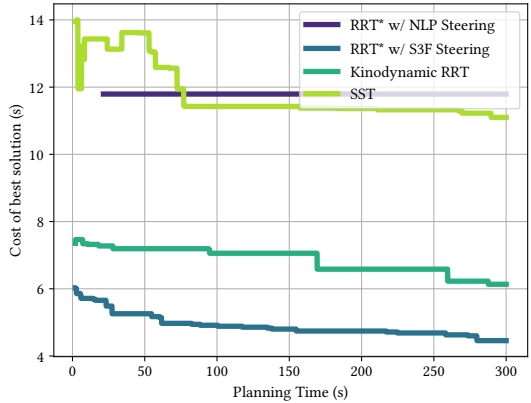}}
    \caption{Dubin's Car with Acceleration average solution cost vs runtime}
    \label{f1tenth_result}
  \end{subfigure}%
  \hspace{0.5cm}
  ~
  \centering
  \begin{subfigure}[t]{3.5in}
    \centering
    \scalebox{0.5}{\input{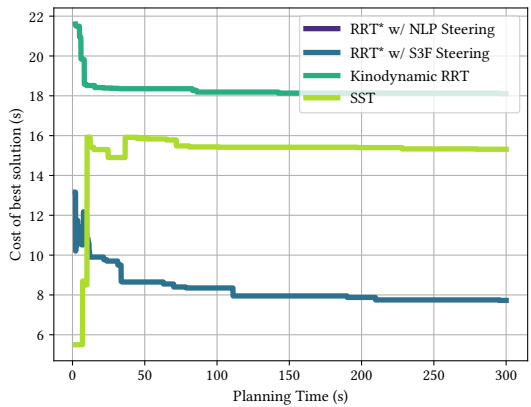}}
    \caption{Tractor Trailer average solution cost vs runtime}
    \label{tractor_trailer_result}
  \end{subfigure}%
  \hspace{0.5cm}
  ~
  \centering
  \begin{subfigure}[t]{3.5in}
    \centering
    \scalebox{0.5}{\input{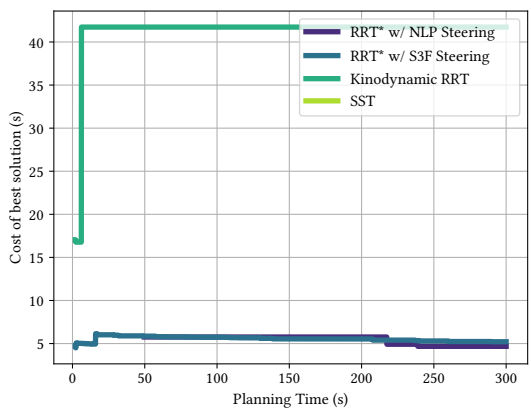}}
    \caption{Quadrotor average solution cost vs runtime}
    \label{quadrotor_result}
  \end{subfigure}

  \caption{Comparison of the planning results of the S3F-RRT*, NLP-RRT*, RRT and SST planning algorithms on the three robot domains. Planning time is plotted against the cost of the best solution found thus far, averaged across 25 planning trials.}
  \label{planning_charts}
  \Description{Comparison of the planning results of the S3F-RRT*, NLP-RRT*, RRT and SST planning algorithms on the three robot domains. Planning time is plotted against the cost of the best solution found thus far, averaged across 25 planning trials.}
\end{figure}



Here we compare planning using the S3F-RRT* algorithm against RRT* with NLP steering, RRT, and SST. By comparing against SST, we can omit a comparison against AO-RRT since previous work~\citep{littlefield2018efficient}\citep{xie2020batch} has shown that empirically SST outperforms AO-RRT. Comparisons are done on all three problem spaces. Starting and ending points for each planning query are sampled randomly across five different maps. 

Figures \ref{f1tenth_result}, \ref{tractor_trailer_result}, and \ref{quadrotor_result} plot the average cost of best solution found by each of the algorithms against wall-clock time for the different robot domains. Results of $25$ planning problems are depicted in each plot. In many cases, it takes the algorithms quite a long time to find their first solution. This causes the graphs to not be monotonically decreasing, since the cost of best solution before a solution is found cannot be plotted. We observe in the graphs that S3F-RRT* is able to find solutions very quickly, and is able to find better solutions than the baseline algorithms irrespective of the amount of computation time given. One of the key reasons why this occurs is that due to the speed of evaluation of the learned steering function, many more RRT* iterations can be completed in a unit time as opposed to NLP-RRT*, enabling the more rapid exploration of the state space by the sampling-based planning algorithm. Furthermore, because S3F does a good job at approximating the optimal steering function, waypoints in the final planned path are connected in a near-optimal fashion. This is something that the baseline algorithms like SST and RRT are unable to do, because in these algorithms waypoints are connected by randomly sampled trajectories, resulting in significant suboptimality. 

\begin{figure}
\begin{tabular}{|l|r|r|r|r|r|r|r| }
  \hline
  & \multicolumn{2}{c|}{Dubin's Car} & 
    \multicolumn{2}{c|}{Tractor Trailer} &
    \multicolumn{2}{c|}{Quadrotor} \\
  & f ($\%$) & t ($s$) & f ($\%$) & t ($s$) & f ($\%$) & t ($s$) \\
  \hline
  RRT & 4 & 0.251 & 0 & 0.083 & 10 & 0.789 \\
  S3F-RRT$^*$ & 20 & 0.480 & 28 & 1.910 & 0 & 16.386 \\
  NLP-RRT$^*$ & 92 & 21.307 & 100 & -- & 70 & 168.656 \\
  SST & 12 & 10.013 & 48 & 11.075 & 100 & -- \\
  \hline
\end{tabular}
\vspace{-3mm}
\caption{Failure rate (f) and time to first solution (t) of different planners}
\label{table}
\vspace{-0.5cm}
\Description{Failure rate (f) and time to first solution (t) of different planners}
\end{figure}

Figure \ref{table} depicts the rate of failure and average time to first solution of the different algorithms. The time to first solution differs from the cost of best solution in Figure \ref{planning_charts} in that the former only considers how long it takes to find the first feasible solution. We can see that across the different problem spaces, S3F-RRT* has lower rates of failure than SST and NLP-RRT*. Figure \ref{quadrotor_result} seems to show that S3F-RRT* and NLP-RRT* have similar performance on the quadrotor domain, but the data in the table shows that S3F-RRT* has a much lower rate of failure and finds its first solution far more quickly, demonstrating that S3F-RRT* indeed has better performance. S3F-RRT* on average is able to find its first solution almost as quickly as RRT. It takes on average an order of magnitude more time for SST and NLP-RRT* to find their first solutions. 


\section{Conclusion}

We introduced State Supervised Steering Function, a learning based approximation of the optimal steering function for complex kinodynamic systems. We demonstrate that the learned steering function can be used in sampling-based planners to achieve superior planning results. This superiority is assessed on metrics of time to find solution and quality of solution for three challenging robot domains. Finally, we present a proof of probabilistic completeness of RRT* using S3F, demonstrating its theoretical soundness.


\balance


\begin{acks}
This work has taken place in the Autonomous Mobile
Robotics Laboratory (AMRL) at UT Austin. AMRL research is supported in part by
NSF (CAREER-2046955, IIS-1954778, SHF-2006404), ARO
(W911NF-19-2-0333,W911NF-21-20217), 
DARPA (HR001120C0031), Amazon, JP Morgan, and Northrop Grumman Mission Systems.
The views and conclusions contained in this document are those of the authors
alone.

\end{acks}

\bibliographystyle{ACM-Reference-Format} 
\bibliography{paper}


\newpage
\section{Supplementary Materials}

\subsection{Robot State/Control Space Bounds}
The following are the bounds of the state and control variables for the Dubin's Car with Acceleration 
robot domain:
\begin{align*}
  x&: [-5, 5] m & y&: [-5, 5] m \\
  \theta&: [0, 2\pi] rad & v&: [-3, 3] \frac{m}{s} \\
  k&: [-1, 1] m^{-1} & a&: [-1, 1] \frac{m}{s^2}
\end{align*}
The following are the bounds of the state and control variables for the Tractor Trailer 
robot domain:
\begin{align*}
  x&: [-5, 5] m & y&: [-5, 5] m \\
  \theta&: [0, 2\pi] rad & v&: [-1, 1] \frac{m}{s} \\
  \alpha&: [0, 2\pi] rad & L&: 0.25 m \\
  D&: 0.5 m & a&: [-1, 1] \frac{m}{s^2} \\
  \phi&: [\tan^{-1}(-L), \tan^{-1}(L)] rad
\end{align*}
The following are the bounds of the state and control variables for the Quadrotor 
robot domain:
\begin{align*}
  x&: [-5, 5] m & y&: [-5, 5] m \\
  z&: [0, 5] m & \dot{x}&: [-3, 3] \frac{m}{s} \\
  \dot{y}&: [-3, 3] \frac{m}{s} & \dot{z}&: [-1, 1] \frac{m}{s} \\
  \theta&: [-\frac{\pi}{2}, \frac{\pi}{2}] rad & \phi&: [-\frac{\pi}{2}, \frac{\pi}{2}] rad \\
  \gamma&: [-\pi, \pi] rad & \dot{\theta}&: [-\pi, \pi] \frac{rad}{s} \\
  \dot{\phi}&: [-\pi, \pi] \frac{rad}{s} & \dot{\gamma}&: [-\frac{\pi}{2}, \frac{\pi}{2}] \frac{rad}{s} \\
  w&: 1.2 kg & L&: 0.3 m \\
  r&: 0.1 m & b&: 0.0245 \\
  \tau_1&: [1.994, 10.095] N & \tau_2&: [1.994, 10.095] N \\
  \tau_3&: [1.994, 10.095] N & \tau_4&: [1.994, 10.095] N
\end{align*}

\subsection{Implementation Details}
All of the experiments were run on a Parallels Desktop virtual machine running Ubuntu ARM64 
on a 2020 M1 Macbook Air. The virtual machine was equipped with $4$ processing cores and $4$ GB RAM. 

For the planning experiments, the S3F-RRT*, NLP-RRT*, and RRT algorithms were implemented in C++ by the 
authors. The Open Motion Planning Library (OMPL) was 
used for the implementation of the SST algorithm. For training dataset generation and in NLP-RRT*, the 
PSOPT optimal control library was used as the NLP solver. 

The policy $\pi$ in S3F was represented as a feedforward neural network. A two hidden layer $256$ 
neuron network with $\tanh$ activations was used for both the Dubin’s car with 
acceleration and tractor trailer problem spaces. A three hidden layer $256$ neuron network with 
the same activations was used for the quadrotor problem space.

\subsection{Probabilistic Completeness Proof}

Here we present a proof of probabilistic completeness (PC) of the S3F-RRT* algorithm. S3F-RRT* is a 
	modification of the original RRT* algorithm designed to make use of a 
	learned steering function. The proof largely follows the structure of the proof of probabilistic
	completeness of geometric RRT, though significant modifications have 
	been made to take into account the presence of kinodynamic constraints and the use of a learned steering 
	function. 

  Let $c^*(x_a, x_b)$ denote the cost of the optimal trajectory from 
  $x_a$ to $x_b$, or equivalently the kinodynamic distance from $x_a$ to $x_b$. We assume that $c^*$ obeys the triangle inequality, that is, $c^*(x_a, x_b) \leq c^*(x_a, x) + c^*(x, x_b)$ for all 
  $x \in X$. Let $\tilde{S}$ be a learned steering function. We assume that with nonzero probability $p$, $\tilde{S}(x_a, x_b)$ yields 
  a state function $\tilde{\Gamma}$ that satisfies $c^*(\tilde{\Gamma}(t), x_b) \leq c^*(x_a, x_b)$ for all $t \in [0, t_f]$. This 
  assumption in essence states that every state along the path produced by $\tilde{S}$ is kinodynamically closer to the goal state than the start state 
  is. For a steering function trained to be optimal, this is a reasonable assumption.    
  
  We will use $B_r(x)$ to denote the subset of the state space $X$ defined by $\{x'|c^*(x', x) \leq r\}$. For simplicity, we assume 
  that there exist $\delta_\mathrm{goal} > 0, x_\mathrm{goal} \in X_\mathrm{goal}$ such that 
  $B_{\delta_\mathrm{goal}}(x_\mathrm{goal}) \subseteq X_\mathrm{goal}$. We denote this simplified goal region $B_{\delta_\mathrm{goal}}(x_\mathrm{goal})$ as $X_\mathrm{goal}^*$. The goal of the motion planning problem is to find a kinodynamically 
  feasible path $\pi : [0, t_\pi] \rightarrow X_\mathrm{free}$ such that $\pi(0) = x_\mathrm{init}$ and $\pi(t_\pi) \in X_\mathrm{goal}^*$. 
  The clearance of $\pi$ is the maximal $\delta_\mathrm{clear}$ such that $B_{\delta_\mathrm{clear}}(\pi(t)) \in X_\mathrm{free}$ for 
  all $t \in [0, t_\pi]$.

  We assume for this proof that there exists a valid trajectory $\pi : [0, t_\pi] \rightarrow 
  X_\mathrm{free}$ with clearance $\delta_\mathrm{clear} > 0$. Without loss of generality, assume 
  that $\pi(t_\pi) = x_\mathrm{goal}$, i.e., the trajectory terminates at the center of the goal 
  region. Let $L$ be the total cost of $\pi$, and let $v = min(\delta_\mathrm{clear}, 
  \delta_\mathrm{goal})$. Let $m = \frac{3L}{v}$. Define a sequence of $m+1$ points $x_0 = 
  x_\mathrm{init}, ..., x_m = x_\mathrm{goal}$ along $\pi$ such that the cost of traversal from one 
  point to the next is $\frac{v}{3}$. Therefore, $c^*(x_i, x_{i+1}) \leq \frac{v}{3}$ for every $0 
  \leq i < m$. We will now prove that as the number of iterations increases, the S3F-RRT* algorithm will generate a 
  path passing through the vicinity of these $m+1$ points with probability asymptotically approaching one. 

  \begin{lemma}
    \label{lemma1}
      Suppose that S3F-RRT* has reached $B_\frac{v}{3}(x_i)$, that is, its tree contains a vertex $x_i'$ such that $x_i' \in B_\frac{v}{3}(x_i)$. If $x_\mathrm{rand} \in B_\frac{v}{3}(x_{i+1})$ and $c^*(x_i, x_\mathrm{rand}) \leq \frac{v}{3}$ (equivalently $x_i \in B_\frac{v}{3}(x_\mathrm{rand})$), then the path from the nearest neighbor $x_\mathrm{near}$ to $x_\mathrm{rand}$ lies entirely in $X_\mathrm{free}$ with probability $p$.
    \end{lemma}
    
    \begin{proof}
      Because $x_\mathrm{near}$ is the nearest neighbor, it is true that $c^*(x_\mathrm{near}, x_\mathrm{rand}) \leq c^*(x_i', x_\mathrm{rand})$. Invoking the triangle inequality,
      \begin{align*}
        c^*(x_\mathrm{near}, x_{i+1}) &\leq c^*(x_\mathrm{near}, x_\mathrm{rand}) + c^*(x_\mathrm{rand}, x_{i+1}) \\
        &\leq c^*(x_i', x_\mathrm{rand}) + c^*(x_\mathrm{rand}, x_{i+1}) \\
        &\leq c^*(x_i', x_i) + c^*(x_i, x_\mathrm{rand}) + c^*(x_\mathrm{rand}, x_{i+1}) \\
        &\leq 3\frac{v}{3} = v
      \end{align*}
      Thus $x_\mathrm{near} \in B_v(x_{i+1})$, meaning $x_\mathrm{near} \in X_\mathrm{free}$. Assume that $c^*(\tilde{\Gamma}(t), x_b) \leq c^*(x_a, x_b)$. The probability that this occurs is $p$. Since each state along $\tilde{\Gamma}$ is closer or as close to $x_\mathrm{rand}$ as $x_\mathrm{near}$, the same logic that was applied above to $x_\mathrm{near}$ can be applied to each respective state. Thus, with probability $p$, the path from $x_\mathrm{near}$ to $x_\mathrm{rand}$ will lie entirely in $x_\mathrm{free}$. 
    \end{proof}
    
    \begin{theorem}
      The probability that S3F-RRT* fails to reach $X_\mathrm{goal}^*$ from $x_\mathrm{init}$ after $k$ iterations is at most $ae^{-bk}$, for some constants $a, b \in \mathbb{R}_{>0}$. 
    \end{theorem}
    
    \begin{proof}
      Assume that $B_\frac{v}{3}(x_i)$ already contains an S3F-RRT* vertex. Let $r_i$ be the probability that in the next iteration a S3F-RRT* vertex will be added to $B_\frac{v}{3}(x_{i+1})$. Recall that due to lemma \ref{lemma1}, $x_\mathrm{rand} \in B_\frac{v}{3}(x_{i+1})$ and $c^*(x_i, x_\mathrm{rand}) \leq \frac{v}{3}$ implies that the path from $x_\mathrm{near}$ to $x_\mathrm{rand}$ will lie entirely in $X_\mathrm{free}$ with probability $p$. In the S3F-RRT* algorithm, after $x_\mathrm{rand}$ is sampled, all states in $X_\mathrm{near}$ are considered as possible parent states. By the definition of $X_\mathrm{near}$, $x_\mathrm{near}$ is a part of this candidate set. Thus, it is guaranteed that $x_\mathrm{new}$ will be added as a S3F-RRT* vertex with probability greater than or equal to $p$. Assume that the probability that both $x_\mathrm{rand} \in B_\frac{v}{3}(x_{i+1})$ and $c^*(x_i, x_\mathrm{rand}) \leq \frac{v}{3}$ is $\gamma_i > 0$. It is safe to assume that this probability is nonzero because any state along the path produced by $S^*(x_i, x_{i+1})$ satisfies these constraints, and so does any state along the portion of $\pi$ from $x_i$ to $x_{i+1}$. Finally, let the conditional probability that $x_\mathrm{new} \in B_\frac{v}{3}(x_{i+1})$ given that $x_\mathrm{rand} \in B_\frac{v}{3}(x_{i+1})$ and $c^*(x_i, x_\mathrm{rand}) \leq \frac{v}{3}$ be $\kappa_i > 0$. It is again safe to assume that this probability is nonzero because $\tilde{\Gamma}$ closely approximates $\Gamma^*$, meaning $x_\mathrm{new}$ will be close to $x_\mathrm{rand}$. Taking into account these probabilities, we have $r_i = p \gamma_i \kappa_i$. Note that this expression is independent of $k$. 
      
      Let $r$ be the minimum of the probabilities $\{r_i | \forall i (0 \leq i < m)\}$. In order for the S3F-RRT* algorithm to reach $X_\mathrm{goal}^*$ from $x_\mathrm{init}$, a S3F-RRT* vertex must be added to $B_\frac{v}{3}(x_{i+1})$ $m$ times for $0 \leq i < m$. This stochastic process can be defined as a Markov chain. Alternatively, this process can be described as $k$ Bernoulli trials with success probability $r$. The planning problem can be solved after $m$ successful outcomes. Note that the success probability $r$ is an underestimate of the true success probability for each trial, and that it is possible that the process ends after less than $m$ successful outcomes. Defining the problem in such a manner allows us to obtain an upper bound on the probability of failure. 
    
      Next, we bound the probabilty of faiure, that is, the probability that the process does not reach state $m$ after $k$ steps. Let $X_k$ denote the number of successes in $k$ trials, then 
      \begin{align*}
        Pr[X_k < m] &= \sum_{i=0}^{m-1}{\binom{k}{i}r^i(1-r)^{k-i}} \\
        &\leq \sum_{i=0}^{m-1}{\binom{k}{m-1}r^i(1-r)^{k-i}} \\
        &\leq \binom{k}{m-1}\sum_{i=0}^{m-1}{(1-r)}^k \\ 
        &\leq \binom{k}{m-1}\sum_{i=0}^{m-1}{(e^{-\tau})^k} \\
        &= \binom{k}{m-1}me^{-rk} \\ 
        &= \frac{\prod_{i=k-m}^{k}{i}}{(k-1)!}me^{-rk} \\ 
        &\leq \frac{m}{(m-1)!}k^me^{-rk} 
      \end{align*}
      where the second statement is justified since $m << k$, the third statement uses the fact that $r < \frac{1}{2}$, and the fourth statement relies on $(1-r) \leq e^{-\tau}$. As $r, m$ are fixed and independent of $k$, the expression $\frac{1}{(m-1)!}k^mme^{-rk}$ decays to zero exponentially with $k$. Therefore, S3F-RRT* is probabilistically complete.
    \end{proof}



\end{document}